\def\year{2017}\relax
\documentclass[letterpaper]{article} 
\usepackage{aaai18}  
\usepackage{times}  
\usepackage{helvet}  
\usepackage{courier}  
\usepackage{url}  
\usepackage{graphicx}  
\usepackage{amsmath,amssymb,graphicx,mathabx,mathtools}
\newtheorem{theorem}{Theorem}
\newtheorem{lemma}{Lemma}

\newtheorem{definition}{Definition}

\renewcommand{\phi}{\varphi}
\renewcommand{\epsilon}{\varepsilon}

\newenvironment{proof}{\noindent{\sf Proof.}}{\hfill $\boxtimes\hspace{2mm}$\linebreak}
\newcommand{\qed}{\hfill $\boxtimes\hspace{1mm}$}

\begin{document}
%
\title{Perfect Recall and Navigation Strategies}
\title{Armstrong's Axioms, Perfect Recall and Navigation Strategies}
\title{Navigability with Imperfect Information}
\title{Navigation with Imperfect Information}
\title{Armstrong's Axioms and Navigation with Imperfect Information}
\title{Armstrong's Axioms and Perfect Recall}
\title{Armstrong's Axioms and Navigation Strategies}








\author{Kaya Deuser \and Pavel Naumov\\ 
Vassar College\\
124 Raymond Avenue\\ 
Poughkeepsie, NY 12604\\
\{kdeuser, pnaumov\}@vassar.edu
}


\maketitle

\begin{abstract}
The paper investigates navigability with imperfect information. It shows that the properties of navigability with perfect recall are exactly those captured by Armstrong's axioms from database theory. If the assumption of perfect recall is omitted, then Armstrong's transitivity axiom is not valid, but it can be replaced by a weaker principle. The main technical results are soundness and completeness theorems for the logical systems describing properties of navigability with and without perfect recall.   
\end{abstract}

\section{Introduction}

Navigation is a commonly encountered task by autonomous agents that need to reach a destination or, more generally, to find a solution to a problem, where the solution is a sequence of instructions that transition a system from one state to another. This task is often performed when the agent does not have precise information about her current location. Examples of such agents are self-navi\-gating missiles, self-driving cars, and robotic vacuum cleaners.  

 Figure~\ref{intro-example figure} depicts an example $T_0$ of a transition system. This system consists of eight states $a,\dots,h$ represented by the vertices of the graph. The agent cannot distinguish state $a$ from state $b$ and state $c$ from state $d$, which is denoted by dashed lines connecting the indistinguishable states. The directed edges of the graph represent transitions that the system can make and the labels on these edges represent the instructions that the agent must give to do this. For example, if in state $a$ the agent executes instruction $0$, then the system transitions into state $g$, if, instead, the agent executes instruction $1$, the system transitions into state $e$. Although in this paper we consider non-deterministic transition systems where the execution of the same instruction can transition the system into one of the several states, for the sake of simplicity the transition system $T_0$ is deterministic.

\begin{figure}[ht]
\begin{center}
\vspace{0mm}
\scalebox{.7}{\includegraphics{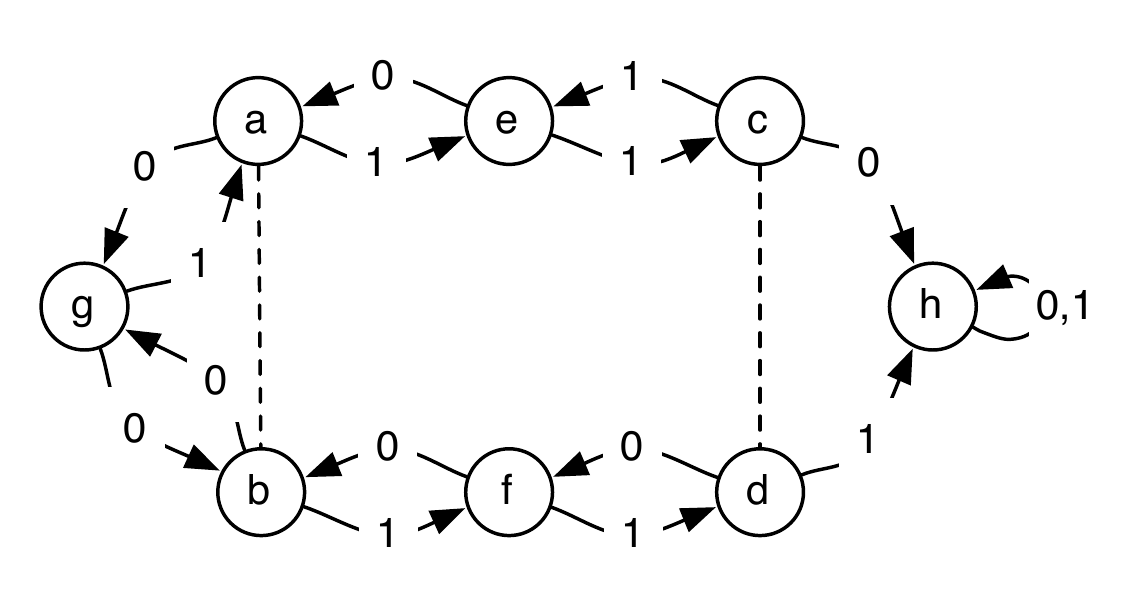}}
\caption{Transition System $T_0$.}\label{intro-example figure}
\end{center}
\vspace{0mm}
\end{figure}
Note that, in system $T_0$ the agent can navigate from state $a$ to state $c$ by using instruction $1$ in state $a$ and the same instruction $1$ again in state $e$. However, a different sequence of instructions is required to reach state $c$ from state $b$. As the agent cannot distinguish state $a$ from state $b$, in state $a$ she does not know which instructions to use to accomplish her goal. Moreover, if the agent does reach state $c$, she cannot verify that the task is completed, because she cannot distinguish state $c$ from state $d$. For this reason, in this paper instead of navigation between states, we consider navigation between equivalence classes of states with respect to the indistinguishability relation. For example, the agent can navigate from class $[a]=\{a,b\}$ to class $[c]=\{c,d\}$ by using instruction $1$ in each state she passes.

\subsubsection{Perfect Recall}
In order to achieve a goal, the agent would need to follow a certain strategy that must be stored in her memory. We assume that the strategy is permanently stored (``hardwired") in the memory and cannot be changed during the navigation. For example, a robotic vacuum cleaner might be programmed to change direction when it encounters a wall, to make a circle when the dirt sensor is triggered, and to follow a straight path otherwise. A crucial  question for us is if the vacuum cleaner can remember the walls and the dirty spots it has previously encountered. In other words, we distinguish an agent that can keep track of the classes of states she visited and the instructions she used from an agent who only knows her current state. We say that in the former case the agent has {\em perfect recall} and in the later she does not. 

A strategy of an agent without perfect recall can only use information available to her  about {\em the current state} to decide which instruction to use. In other words, a strategy of such an agent is a function that maps classes of indistinguishable {\em states} into instructions. A strategy of an agent with perfect recall can use information about {\em the history} of previous transitions to decide which instruction to use. In other words, a strategy of such an agent is a function that maps classes of indistinguishable {\em histories} into instructions. We call the former {\em memoryless} strategies and the later {\em recall} strategies.

In theory, a robotic vacuum cleaner without perfect recall is only equipped with read-only memory to store the strategy. A theoretical
robotic vacuum with perfect recall in addition to read-only memory that contains the strategy also has an unlimited read-write memory that contains logs of all previous transitions. In practice, the most popular brand of robotic vacuum cleaners, Roomba, is only using read-write memory to store information, such as a cleaning schedule, that is not used for navigation. This means that Roomba is using a memoryless strategy. The other popular robotic vacuum cleaner, Neato, is scanning the room before cleaning to use this information for navigation purposes. Although, of course, Neato only has a read-write memory of a limited size, the navigation strategy used by Neato is an example of a recall strategy. 

\subsubsection{Examples}
Transition system $T_0$ has a recall strategy to navigate from class $[a]$ to class $[e]$. This might be surprising, because the same strategy must work to navigate from both state $a$ and state $b$ of class $[a]$ to the only state $e$ of class $[e]$. Yet, one would expect to use instruction $1$ in state $a$ and instruction $0$ in state $b$ to get to $e$. The recall strategy to navigate from class $[a]$ to class $[e]$ consists in first using instruction $0$ once no matter what the starting state is, and then using instruction $1$ until state $e$ is reached. When this strategy is used starting from a state in class $[a]$, the system first transitions into state $g$, then into state $a$, and then finally into state $e$. This is a recall strategy because it uses a different instruction in state $a$ depending if the system has already visited a state in class $[g]$ or not. 

To show that there is no memoryless strategy to navigate from class $[a]$ to class $[e]$ note that any such strategy will have to use the same instruction $i_0$ at every visit to states $a$ and $b$. If $i_0=0$, then when the navigation starts in state $b$, the system stays ``locked" in set of states $\{a,g,b\}$ and never reaches state $e$. Similarly, if $i_0=1$, then when the navigation starts in state $b$, the system stays ``locked" in set of states $\{b,f,d,h\}$ and never reaches state $e$. Thus, there is no memoryless strategy to navigate from class $[a]$ to class $[e]$.





In some situation, even when there is a path between appropriate states, there might be no memoryless strategy and no recall strategy. For example, transition system $T_0$ has no recall strategy to navigate from class $[c]$ to class $[g]$. Indeed, if such a recall strategy exists, it would have to use the same instruction $i_0$ when the system {\em starts} in either state $c$ or state $d$. Suppose $i_0=0$. Thus, if the system {\em starts} in state $c$ it is ``locked" in state $h$. Assume now that $i_0=1$. Hence, if the system {\em starts} in state $d$ it is ``locked" in state $h$. Therefore, there is no recall strategy from class $[c]$ to class $[g]$.

\begin{table}[ht]
\centering
\begin{tabular}{ c | c c c c c c}
            & $\{a,b\}$ & $\{c,d\}$ & $\{e\}$   & $\{f\}$   & $\{g\}$ & $\{h\}$\\ \hline
$\{a,b\}$   & m         & m         & r         & r         & m       & r$^\dag$ \\
$\{c,d\}$   & -         & m         & -         & -         & -       & r$^\dag$ \\
$\{e\}$     & m         & m         & m         & r$^\dag$  & m       & m \\
$\{f\}$     & m         & m         & r$^\dag$  & m         & m       & m \\
$\{g\}$     & m         & m         & m         & m         & m       & m$^\dag$ \\
$\{h\}$     & -         & -         & -         & -         & -       & m\\

\end{tabular}
\caption{Navigability between classes in system $T_0$.}\label{s r table}
\end{table}


Table~\ref{s r table} shows when memoryless strategies and recall strategies between classes of states of transition system $T_0$ exist. In this table, letter ``m"  at the intersection of row $x$ and column $y$ denotes the existence of a memoryless strategy from class $x$ to class $y$. Letter ``r" denotes the existence of a recall strategy, but no memoryless strategy, and dash ``-" denotes that neither memoryless strategy nor recall strategy exist. Symbol $\dag$ marks the cases that we have  found to be interesting to think about.

\subsubsection{Navigability between Sets}
As we have seen, there is a recall strategy, but no memoryless strategy, to navigate from class $[a]$ to class $[e]$. One can similarly show that there is no memoryless strategy to navigate from class $[a]$ to class $[f]$. However, if the goal is to navigate from class $[a]$ to either class $[e]$ or to class $[f]$, then there is a memoryless strategy to do this. Indeed, consider a memoryless strategy that uses instruction $1$ in every state. This strategy can transition the system from a state of class $[a]$ to a state of a class in set $\{[e],[f]\}$. 

Thus, navigability between {\em sets of classes} can not be reduced to navigability between {\em classes}. For this reason, in this paper we study properties of navigability between sets of classes. If there is a strategy to navigate from a set of classes $A$ to a set of classes $B$, then we write $A\rhd B$. It will be clear from the context if we refer to the existence of a memoryless strategy or a recall strategy.

\subsubsection{Universal Properties of Navigability}
In the examples above we talked about properties of navigability for the transition system $T_0$. In the rest of this paper we study universal properties of navigability between sets of classes that are true in all transition systems. An example of such a property is {\em reflexivity}: $A\rhd B$, where $A\subseteq B$. This property is true for both memoryless and recall strategies because absolutely {\em any} strategy can be used to navigate from a subset to the whole set. In fact, in this case the goal is achieved before the navigation even starts.  

Another example of a property of navigation which is universally true for both memoryless and recall strategies is {\em augmentation}: $A\rhd B \to (A\cup C) \rhd (B\cup C)$. It says that if there is a strategy to navigate from set $A$ to set $B$, then there is a strategy to navigate from set $A\cup C$ to set $B\cup C$.

An example of a property which is universally true for recall strategies, but is not universally true for memoryless strategies is {\em transitivity}: $A\rhd B\to(B\rhd C\to A\rhd C)$. It states that if there is a strategy to navigate from set $A$ to set $B$ and a strategy to navigate from set $B$ to set $C$, then there is a strategy to navigate from set $A$ to set $C$. To see that this property is not universally true for memoryless strategies, note that, in transition system $T_0$,  memoryless strategy that always uses instruction $0$ can be used to navigate from set $\{[a]\}$ to set $\{[g]\}$ and memoryless strategy that always uses instruction $1$ can be used to navigate from set $\{[g]\}$ to set $\{[e]\}$. At the same time, as we have shown earlier, there is no memoryless strategy to navigate from set $\{[a]\}$ to set $\{[e]\}$. 

In this paper we show that reflexivity, augmentation, and transitivity principles form a sound and complete logical system that describes all universal properties of navigability by recall strategies. These are the three principles known in database theory as Armstrong's axioms~\cite[p.~81]{guw09}, where they give a sound and complete axiomatization of functional dependency~\cite{a74}. We also give a sound and complete axiomatization of universal properties of navigability by memoryless strategies. It consists of the reflexivity and augmentation principles mentioned above as well as the {\em monotonicity} principle $(A\cup C)\rhd B\to A\rhd B$. The latter principle is true for the recall strategies as well, but it is provable from Armstrong's axioms.


\subsubsection{Literature Review}
Most of the existing literature on logical systems for reasoning about strategies is focused on modal logics. Logics of coalition power were developed by \cite{p01illc,p02}, who also proved the completeness of the basic logic of coalition power. Pauly's approach has been widely studied in literature~\cite{g01tark,vw05ai,b07ijcai,sgvw06aamas,abvs10jal,avw09ai,b14sr}. Alternative logical system  were proposed by \cite{mn12tocl}, \cite{w15lori,w17synthese}, and \cite{lw17icla}.  
\cite{ahk02} introduced Alternating-Time Temporal Logic (ATL) that combines temporal and coalition modalities. \cite{vw03sl} proposed to combine ATL with epistemic modality to form Alternating-Time Temporal Epistemic Logic. A completeness theorem for a logical system that combines coalition power and epistemic modalities was proven by \cite{aa12aamas}. 

The notion of a strategy that we consider in this paper is much more restrictive than the notion of strategy in the works mentioned above. Namely, we assume that the strategy must be based only on the information available to the agent. This is captured in our setting by requiring the strategy to be the same in all indistinguishable states or histories.  This restriction on strategies has been studied before under different names. \cite{ja07jancl} talk about ``knowledge to identify and execute a strategy",  \cite{jv04fm} discuss ``difference between an agent knowing that he has a suitable strategy and knowing the strategy itself". \cite{v01ber} calls such strategies ``uniform". \cite{nt17aamas} use the term ``executable strategy". \cite{nt17tark} proposed a complete trimodal logical system describing an interplay between distributed knowledge, uniform strategic power modality, and standard strategic power modality for achieving a goal by a coalition in one step. \cite{fhlw17ijcai} developed a complete logical system in a single-agent setting for uniform strategies to {achieve} a goal in multiple steps. \cite{nt17aamas} developed a similar system  for maintaining a goal in multi-agent setting. Our contribution is different from all of the above papers by being the first to propose complete logical systems for recall strategies and memoryless strategies.

\subsubsection{Paper Outline} 
In the next section we define transition systems and the syntax of our logical systems. This section applies equally to recall and memoryless strategies. The rest of the paper is split into two independent sections. The first of them proves the soundness and the completeness of Armstrong's axioms for navigability under recall strategies and the second gives an axiomatization for memoryless strategies. 

\section{Syntax and Semantics}

In this section we formally define the language of our logical system, the notion of a transition system, and the related terminology. In the introduction, relation $\rhd$ was viewed as a relation between equivalence classes of a given transition system. Thus, our language depends on these classes and changes from transition system to transition system. In order to have a single language for all transition systems we introduce a fixed finite set of ``views" $V$, whose elements act as names of the equivalence classes in any given transition system. 

\begin{definition}\label{Phi}
$\Phi$ is the minimal set of formulae such that
\begin{enumerate}
    \item $A\rhd B\in\Phi$ for all nonempty\footnote{If one allows sets $A$ and $B$ to be empty, most of the proofs in this paper will remain unchanged, but both logical systems will need an additional axiom $\neg(A\rhd \varnothing)$ for each nonempty set $A$.} sets $A,B\subseteq V$,
    \item $\neg\phi,\phi\to\psi\in \Phi$ for all formulae $\phi,\psi\in \Phi$.
\end{enumerate}
\end{definition}
Each transition system specifies a mapping $*$ of views into equivalence classes of states. Transitions between states under an instruction $i$ are captured by a transition function $\Delta_i$.
\begin{definition}\label{transition system}
$(S,\sim,*,I,\{\Delta_i\}_{i\in I})$ is a transition system, if
\begin{enumerate}
    \item $S$ is a set of states,
    \item $\sim$ is an equivalence (indistinguishability) relation on $S$,
    \item $*$ is a function from $V$ to $S/\!\sim$,
    \item $I$ is an arbitrary nonempty set of ``instructions",
    \item $\Delta_i$ maps set $S$ into nonempty subsets of $S$ for each $i\in I$.
\end{enumerate}
\end{definition}
We write $a^*$ instead of $*(a)$, where $a\in V$. An example of a transition system is system $T_0$ depicted in Figure~\ref{intro-example figure}.

\begin{definition}\label{history}
A finite sequence $w_0,i_1,w_1,\dots,i_n,w_n$, where $n\ge 0$,
is called a history if
\begin{enumerate}
    \item $w_k\in S$ for each $k$ such that $0\le k\le n$,
    \item $i_k\in I$, for each $k$ such that $1\le k\le n$,
    \item $w_k\in \Delta_{i_k}(w_{k-1})$,
    for each $k$ such that $1\le k\le n$.
\end{enumerate}
\end{definition}
For example, sequence $g,1,a,1,e,1,c,0,h$ is a history for system $T_0$.
The set of all histories is denoted by $H$.

\begin{definition}
History $h=w_0,i_1,w_1,\dots,i_n,w_n$ is indistinguishable from history  $h'=w'_0,i_1,w'_1,\dots,i_n,w'_n$ if $w_k \sim w'_k$ for each $k$ such that $0\le k\le n$.
\end{definition}
For example, histories $a,0,g$ and $b,0,g$ are indistinguishable in transition system $T_0$. Indistinguishability of histories of $h$ and $h'$ is denoted by $h\approx h'$. The equivalence class of history $h$ with respect to this equivalence relation is denoted by $\ldbrack h \rdbrack$. Equivalence class of a state $w$ with respect to equivalence relation $\sim$ is denoted by $[w]$.

\begin{lemma}\label{histories approx lemma}
If $w_0,i_1,w_1,\dots,w_n\approx w'_0,i_1,w'_1,\dots,w'_n$, then $[w_k]=[w'_k]$ for each $k\le n$. \qed
\end{lemma}
 
\begin{definition}\label{strategy}
 A memoryless strategy is a function from set $S/\!\sim$ to set $I$. A recall strategy maps set $H/\!\approx$ to set $I$.
 \end{definition}
We write $s[w]$ and $s\ldbrack h\rdbrack$ instead of $s([w])$ and $s(\ldbrack h\rdbrack)$.
\begin{definition}\label{path}
An infinite sequence $w_0,i_1,w_1,i_2,w_2\dots$ is called
a path under a memoryless strategy $s$ if for each $k\ge 1$
\begin{enumerate}
    \item $w_0,i_1,w_1,i_2,w_2\dots,w_{k-1}\in H$,
    \item $i_{k}=s[w_{k-1}]$.
\end{enumerate}
\end{definition}

\begin{lemma}\label{path exists}
For any history $w_0,i_1,w_1,\dots, i_n,w_n$ and any memoryless strategy $s$, if $i_{k}=s[w_{k-1}]$ for each $k$ such that $1\le k\le n$, then there are states $w_{n+1},w_{n+2},\dots$ and instructions $i_{n+1},i_{n+2},\dots$ such that sequence $w_0,i_1,w_1,\dots, i_n,w_n,i_{n+1},w_{n+1},i_{n+2},w_{n+2},\dots$ is a path under strategy $s$. 
\end{lemma}
\begin{proof} Elements $i_{n+1},w_{n+1},i_{n+2},w_{n+2},\dots$ can be constructed recursively because (a) there is a state $w_{k+1}\in\Delta_{i_{k+1}}(w_k)$ for any state $w_{k}$ and any $i_{k+1}\in I$ by item 5 of Definition~\ref{transition system}; (b) $I\neq\varnothing$ by item 4 of Definition~\ref{transition system}.
\end{proof}

\begin{definition}\label{perfect path}
An infinite sequence $w_0,i_1,w_1,i_2,w_2\dots$ is called
a path under a recall strategy $s$ if for each $k\ge 1$
\begin{enumerate}
    \item $w_0,i_1,w_1,i_2,w_2\dots,w_{k-1}\in H$,
    \item $i_{k}=s\ldbrack w_0,i_1,w_1,i_2,w_2\dots,w_{k-1}\rdbrack$.
\end{enumerate}
\end{definition}

\begin{definition}\label{*set}
$A^*=\{a^*\;|\;a\in A\}$, for all sets $A\subseteq V$.
\end{definition}

\begin{definition}\label{path set}
For a given memoryless strategy or recall strategy $s$, let $Path_s(A)$ be the set of all paths $w_0,i_1,w_1,i_2,w_2\dots$ under $s$ such that $[w_0]\in A^*$.
\end{definition}

\begin{definition}\label{visit set}
Let set $Visit_s(B)$ be the set of all paths $w_0,i_1,w_1,i_2,w_2\dots$ under $s$ such that $[w_k]\in B^*$ for some $k\ge 0$.
\end{definition}
We write $Visit(B)$ instead of $Visit_s(B)$ when value of $s$ is clear from the context.

\section{Navigation with Recall Strategies}

In this section we show that Armstrong's axioms give a complete axiomatization of navigability between sets of classes with recall strategies. We start with a formal semantics of navigability relation $\rhd$ under recall strategies.

\begin{definition}\label{perfect sat}
$T\vDash A\rhd B$ if  $Path_s(A)\subseteq Visit(B)$ for some recall strategy $s$ of transition system $T$. 
\end{definition}

\subsubsection{Axioms} 
The axioms of the logical system that we consider in this section are the tautologies in language $\Phi$ and the following additional principles known as Armstrong's axioms \cite[p.~81]{guw09}: 

\begin{enumerate}
    \item Reflexivity: $A\rhd B$, where $A\subseteq B$,
    \item Augmentation: $A\rhd B \to (A\cup C)\rhd (B\cup C)$,
    \item Transitivity: $A\rhd B\to (B\rhd C \to A\rhd C)$.
\end{enumerate}
We write $\vdash\phi$ if formula $\phi$ is provable from these axioms using the Modus Ponens inference rule. We write $X\vdash\phi$ if $\phi$ is provable using a set of additional axioms $X$. 

\subsubsection{Soundness} In this section we prove soundness of the above axioms with respect to perfect recall semantics. 

\begin{theorem}\label{perfect soundness}
If $\vdash\phi$, then $T\vDash\phi$ for every system $T$.
\end{theorem}

The soundness of Armstrong's axioms with respect to perfect recall semantics is established in Lemma~\ref{perfect reflexivity}, Lemma~\ref{perfect augmentation}, and Lemma~\ref{perfect transitivity} below. We start with technical notions that we need to prove Lemma~\ref{perfect transitivity}.

\begin{definition}\label{truncation}
For any history $w_0,i_1,\dots,i_n,w_n$ and any set $B\subseteq V$, let truncation $(w_0,i_1,\dots,i_n,w_n)|_B$ be 
\begin{enumerate}
    \item $w_0,i_1,\dots,i_n,w_n$, if $n=0$ or $[w_0]\in B^*$,
    \item $(w_1,i_2,\dots,i_n,w_n)|_B$, otherwise.
\end{enumerate}
\end{definition}
For transition system $T_0$, if $B=\{b\}$ and $b^*=[e]$, then truncation $(g,1,a,1,e,1,c,0,h)|_B$ is history $(e,1,c,0,h)$.

\begin{lemma}\label{histories truncate lemma}
If $h_1\approx h_2$, then $h_1|_B\approx h_2|_B$.\qed
\end{lemma}

\begin{definition}\label{circ}
For any recall strategies $s_1$ and $s_2$, any set $B\subseteq V$, and any history $h=w_0,i_1,\dots,w_n$, let
$$
(s_1\circ_B s_2)\ldbrack h\rdbrack=
\begin{cases}
s_1\ldbrack h\rdbrack, & \mbox{ if } \forall k\le n\,([w_k]\notin B^*) ,\\
s_2\ldbrack h|_B\rdbrack, & \mbox{ otherwise}.
\end{cases}
$$
\end{definition}
In other words, ``composition" strategy $s_1\circ_B s_2$ follows strategy $s_1$ until a state $w$ is reached such that $[w]\in B^*$. Once this happens, the memory of all prior states and instructions is erased and strategy $s_2$ is executed as if navigation started from state $w$.

\begin{lemma}
Operation $s_1\circ_B s_2$ is well-defined.
\end{lemma}
\begin{proof}
See Definition~\ref{truncation}, Lemma~\ref{histories approx lemma}, and Lemma~\ref{histories truncate lemma}.
\end{proof}

\begin{definition}\label{truncation of paths}
For any path $w_0,i_1,w_1,i_2,\dots$ and any set $B\subseteq V$, let truncation $(w_0,i_1,w_1,i_2,\dots)|_B$ be
\begin{enumerate}
    \item path $w_0,i_1,w_1,i_2,\dots$, if $[w_0]\in B^*$,
    \item path $(w_1,i_2,w_2,i_3\dots)|_B$, otherwise.
\end{enumerate}
\end{definition}

\begin{lemma}\label{truncate on B lemma}
For any recall strategies $s_1,s_2$ and any sets $A,B\subseteq V$, if path $\pi\in Path_{s_1\circ_B s_2}(A)\cap Visit(B)$, then $\pi|_B\in Path_{s_2}(B)$. \qed
\end{lemma}

\begin{lemma}\label{strange lemma}
$Path_{s_1\circ_B s_2}(A)\subseteq Visit(B)$
if $Path_{s_1}(A)\subseteq Visit(B)$, for any recall strategies $s_1,s_2$ and $A,B\subseteq V$.
\end{lemma}
\begin{proof}
Assume that there is a path $w_0,i_1,w_1,\dots$ in $Path_{s_1\circ_B s_2}(A)$ such that $w_0,i_1,w_1,\dots\notin Visit(B)$. Thus, $i_k = {(s_1\circ_B s_2)}\ldbrack w_0,i_1,w_1,\dots, w_{k-1}\rdbrack$ for each $k\ge 1$ and $[w_0]\in A^*$ by Definition~\ref{path set} and Definition~\ref{perfect path}. Also, $[w_k]\notin B^*$ for each $k\ge 0$ by Definition~\ref{visit set}. Hence, by Definition~\ref{circ}, we have $i_k = {s_1}\ldbrack w_0,i_1,w_1,\dots, w_{k-1}\rdbrack$ for each $k\ge 1$. Thus, $w_0,i_1,w_1,\dots\in Path_{s_1}(A)$ by Definition~\ref{perfect path} and Definition~\ref{perfect path}. Therefore, $w_0,i_1,w_1,\dots\in Visit(B)$ by the assumption $Path_{s_1}(A)\subseteq Visit(B)$ of the lemma.
\end{proof}
\begin{lemma}\label{apply s2}
For any recall strategies $s_1,s_2$ and $A,B\subseteq V$, if $Path_{s_1\circ_B s_2}(A)\subseteq Visit(B)$, then $(Path_{s_1\circ_B s_2}(A))|_B\subseteq Path_{s_2}(B)$.
\end{lemma}
\begin{proof}
Consider any path $\pi\in Path_{s_1\circ_B s_2}(A)$. Then, $\pi\in Visit(B)$ because $Path_{s_1\circ_B s_2}(A)\subseteq Visit(B)$. Hence, $\pi|_B\in Path_{s_2}(B)$ by Lemma~\ref{truncate on B lemma}.
\end{proof}
\begin{lemma}\label{less is more}
$Path_s(A)\subseteq Visit(C)$, if $(Path_s(A))|_B\subseteq Visit(C)$,
for any $s$ and any  $A,B,C\subseteq V$.\qed
\end{lemma}
\begin{lemma}\label{path properties} 
For any recall strategy $s$ and any $A,B,C\subseteq V$,
\begin{enumerate}
    \item\label{line 1} $Path_s(A)\subseteq Path_s(B)$ if $A\subseteq B$,
    \item\label{line 2} $Path_s(A)\subseteq Visit(A)$,
    \item\label{line 3} $Path_s(A\cup C)=Path_s(A) \cup Path_s(C)$,
    \item\label{line 4} $Visit(B) \cup Visit(C) = Visit(B\cup C)$.
\end{enumerate}
\end{lemma}
\begin{proof}
See Definition~\ref{perfect path} and Definition~\ref{visit set}.
\end{proof}
The next three lemmas prove soundness of Armstrong's axioms with respect to the perfect recall semantics. 
\begin{lemma}\label{perfect reflexivity}
$T\vDash A\rhd B$, where $A\subseteq B$.
\end{lemma}
\begin{proof}
By Definition~\ref{transition system}, set $I$ is nonempty, and thus it contains at least one instruction $i_0$. Let $s$ be a recall strategy such that $s\ldbrack h\rdbrack=i_0$ for each history $h$. By item~\ref{line 1} and item~\ref{line 2} of Lemma~\ref{path properties},
$
Path_s(A)\subseteq Path_s(B)\subseteq Visit(B).
$
Therefore, $T\vDash A\rhd B$ by Definition~\ref{perfect sat}.
\end{proof}
\begin{lemma}\label{perfect augmentation}
If $T\vDash A\rhd B$, then $T\vDash(A\cup C)\rhd(B\cup C)$.
\end{lemma}
\begin{proof}
By Definition~\ref{perfect sat}, the assumption $T\vDash A\rhd B$ implies that there exists a recall strategy $s$ such that $Path_s(A)\subseteq Visit(B)$. Thus, by item~\ref{line 3}, item~\ref{line 2}, and item~\ref{line 4} of Lemma~\ref{path properties},
\begin{eqnarray*}
Path_s(A\cup C)& = & Path_s(A) \cup Path_s(C) \\
&\subseteq &  Visit(B) \cup Path_s(C) \\
&\subseteq & Visit(B) \cup Visit(C)\\
&\subseteq & Visit(B\cup C).
\end{eqnarray*}
Therefore, $T\vDash(A\cup C)\rhd(B\cup C)$ by Definition~\ref{perfect sat}.
\end{proof}

\begin{lemma}\label{perfect transitivity}
If $T\vDash A\rhd B$ and $T\vDash B\rhd C$, then $T\vDash A\rhd C$.
\end{lemma}
\begin{proof}
By the assumption $T\vDash A\rhd B$ and Definition~\ref{perfect sat}, we have $Path_{s_1}(A)\subseteq Visit(B)$ for some recall strategy $s_1$. Similarly, the assumption $T\vDash B\rhd C$ implies that there exists a recall strategy $s_2$ such that $Path_{s_2}(B)\subseteq Visit(C)$.

By Lemma~\ref{strange lemma}, statement $Path_{s_1}(A)\subseteq Visit(B)$ implies $Path_{s_1\circ_B s_2}(A)\subseteq Visit(B)$. 
Thus, by Lemma~\ref{apply s2}, 
$$(Path_{s_1\circ_B s_2}(A))|_B\subseteq Path_{s_2}(B).$$ 
Hence, we have $(Path_{s_1\circ_B s_2}(A))|_B\subseteq Visit(C)$ because of $Path_{s_2}(B)\subseteq Visit(C)$.
Thus, by Lemma~\ref{less is more}, $Path_{s_1\circ_B s_2}(A)\subseteq Visit(C)$. Therefore, $T\vDash A\rhd C$ by Definition~\ref{perfect sat}.
\end{proof}

\subsubsection{Completeness} 
In the rest of this section we prove the completeness of Armstrong's axioms with respect to the perfect recall semantics. We start by defining a canonical transition system $T(X)=(V\cup \{\circlearrowleft\},=,*,I,\{\Delta_i\}_{i\in I})$ for an arbitrary maximal consistent set of formulae $X\subseteq \Phi$. The set of states of this transition system consists of a single state for each view, plus one additional state that we denote by symbol $\circlearrowleft$. Informally, the additional state is a sink or a ``black hole" state from which there is no way out. State $h$ in transition system $T_0$ depicted in Figure~\ref{intro-example figure} is an example of a black hole state. Note that the indistinguishability relation on the states of the canonical transition system is equality relation $=$. That is, the agent has an ability to distinguish any two different states in the system. The fact that equality is suitable as an indistinguishability relation for the canonical transition system with perfect recall is surprising. The indistinguishability relation for the canonical transition system for memoryless strategies, discussed in the next section, is different from equality.  

Each view $v\in V$ is also a state in the canonical transition system. The equivalence class of state $v$ consists of the state itself: $[v]=\{v\}$. We define $v^*$ to be class $[v]$.

\begin{lemma}\label{* iff}
$u\in A$ iff $[u]\in A^*$ for each view $u\in V$ and each set $A\subseteq V$. \qed
\end{lemma}


Informally, if set $X$ contains formula $A\rhd B$, then we want the canonical transition system $T(X)$ to have a recall strategy to navigate from set $A^*=\{[a]\;|\;a\in A\}=\{\{a\}\;|\;a\in A\}$ to set $B^*=\{[b]\;|\;b\in B\}=\{\{b\}\;|\;b\in B\}$. It turns out that it is sufficient to have just a single instruction that transitions the system from any state in set $A$ to a state in set $B$. We denote this instruction by pair $(A,B)$. 

\begin{definition}\label{perfect canonical I}
$I=\{(A,B)\;|\; A\rhd B \in X\}$.
\end{definition}
Recall that assumption $A\rhd B\in X$ requires sets $A$ and $B$ to be nonempty due to Definition~\ref{Phi}.


As discussed above, for any instruction $(A,B)$ we define the nondeterministic transition function $\Delta_{(A,B)}$ to transition the system from a state in $A$ to a state in $B$. If used outside of set of states $A$, instruction $(A,B)$ transitions the system into black hole state $\circlearrowleft$:

\begin{definition}\label{perfect canonical delta}
$$\Delta_{(A,B)}(w)=
\begin{cases}
B, & \mbox{ if $w\in A$},\\
\{\circlearrowleft\}, & \mbox{ otherwise}.
\end{cases}
$$
\end{definition}
This concludes the definition of the transition system $T(X)$.

Next, for any recall strategy $s$ and any set of states $G\subseteq V$, we define a family of sets of states $\{G^s_n\}_{n\ge 0}$. Informally,  set $G^s_n$ is the set of all states from which strategy $s$ ``draws" the system into set $G$ after at most $n$ transitions. For any history $h=w_0,i_1,w_1,\dots,i_n,w_n$, by $hd(h)$ we mean the state $w_n$.
\begin{definition}\label{perfect Gn}
For any recall strategy $s$ and any nonempty set $G\subseteq V$, let chain of $G^s_0\subseteq G^s_1\subseteq\dots\subseteq V\cup\{\circlearrowleft\}$ be defined as
\begin{enumerate}
    \item $G^s_0=G$, 
    \item $G^s_{n+1}=G^s_{n}\cup\left\{hd(h)\;\middle|\; h\in H \mbox{ and }\Delta_{s\ldbrack h\rdbrack}(hd(h))\subseteq G^s_{n}\right\}$ for all $n\ge 0$.
\end{enumerate}
\end{definition}
Note that this definition, in essence, has an existential quantifier over history $h$. Thus, informally, strategy $s$ is allowed to ``manipulate" the history in order to ``draw" the system into set $G$.

\begin{lemma}\label{no holes in G}
$\circlearrowleft\;\notin G^s_n$ for each $n\ge 0$.
\end{lemma}
\begin{proof}
We prove this statement by contradiction. Let $n$ be the smallest non-negative integer number such that $\circlearrowleft\;\in G^s_n$. Note that $n\neq 0$ because $G^s_0=G\subseteq V$ by Definition~\ref{perfect Gn}. Thus, $n>0$. Hence, by Definition~\ref{perfect Gn}, there exists a history $h$ such that $hd(h)=\;\circlearrowleft$ and
$
\Delta_{s\ldbrack h\rdbrack}(\circlearrowleft)\subseteq G^s_{n-1}.
$
Note that $\Delta_{s\ldbrack h\rdbrack}(\circlearrowleft)=\{\circlearrowleft\}$ by Definition~\ref{perfect canonical delta}. Therefore, $\circlearrowleft\;\in G^s_{n-1}$, which contradicts the choice of integer $n$.
\end{proof}
\begin{definition}\label{perfect Ginfty}
$G^s_\infty=\bigcup_{k\ge 0} G^s_k.$
\end{definition}
We now prove properties of the family of sets $\{G^s_n\}_{n\ge 0}$ that are needed to finish the proof of the completeness.
\begin{lemma}\label{g to Gn-1}
$X\vdash \{g\}\rhd G^s_{n-1}$ for each integer $n\ge 1$ and each state $g\in G^s_n$.
\end{lemma}
\begin{proof}
By Definition~\ref{perfect Gn}, assumptions $n\ge 1$ and $g\in G^s_n$ imply that there is a history $h$ such that $hd(h)=g$ and 
\begin{equation}\label{delta g}
  \Delta_{s\ldbrack h\rdbrack}(g)\subseteq G^s_{n-1}. 
\end{equation}
Furthermore, by Definition~\ref{perfect canonical I}, there are nonempty sets $A,B\subseteq V$ such that $s\ldbrack h \rdbrack=(A,B)$. 

\noindent{\em Case I}: $g\notin A$. Then, $\Delta_{s\ldbrack h\rdbrack}(g)=\{\circlearrowleft\}$ by Definition~\ref{perfect canonical delta}. Hence, $\{\circlearrowleft\}\subseteq G^s_{n-1}$ by equation~(\ref{delta g}), which contradicts Lemma~\ref{no holes in G}.

\noindent{\em Case II}: $g\in A$. Then, $\vdash \{g\}\rhd A$ by the Reflexivity axiom. At the same time, $X\vdash A\rhd B$ by Definition~\ref{perfect canonical I}. Thus, by the Transitivity axiom, $X\vdash \{g\}\rhd B$.

Assumption $g\in A$ implies $\Delta_{s\ldbrack h\rdbrack}(g)=B$, by Definition~\ref{perfect canonical delta}. Thus, $B\subseteq G^s_{n-1}$ due to equation~(\ref{delta g}). Hence, $\vdash B\rhd G^s_{n-1}$ by the Reflexivity axiom.

Finally, statements $X\vdash \{g\}\rhd B$ and $\vdash B\rhd G^s_{n-1}$ by the Transitivity axiom imply, using Modus Ponens inference rule twice, that $X \vdash \{g\}\rhd G^s_{n-1}$.
\end{proof}
\begin{lemma}\label{all the a}
For any $n\ge 1$ and any views $a_1,\dots,a_n\in V$, if $X\vdash \{a_k\}\rhd B$ for each $k\le n$, then $X\vdash \{a_1,\dots,a_n\}\rhd B$.
\end{lemma}
\begin{proof}
We prove this statement by induction on $n$. In the base case, $X\vdash \{a_1\}\rhd B$ due to the assumption of the lemma.

By the induction hypothesis, $X\vdash \{a_1,\dots,a_{n-1}\}\rhd B$. Thus, by the Augmentation axiom,
\begin{equation}\label{friday}
    X\vdash \{a_1,\dots,a_{n-1},a_n\}\rhd B\cup \{a_n\}.
\end{equation}
At the same time, $X\vdash \{a_n\}\rhd B$ by the assumption of the lemma. Hence, by the Augmentation axiom $X\vdash B\cup\{a_n\}\rhd B$. Thus, $X\vdash \{a_1,\dots,a_{n-1},a_n\}\rhd B$ by statement~(\ref{friday}) and the Transitivity axiom.
\end{proof}
\begin{lemma}\label{Gn to Gn-1}
$X\vdash G^s_n\rhd G^s_{n-1}$ for each $n\ge 1$.
\end{lemma}
\begin{proof}
The statement of the lemma follows from Lemma~\ref{g to Gn-1} and Lemma~\ref{all the a}.
\end{proof}
\begin{lemma}\label{Gn G}
$X\vdash G^s_n\rhd G$ for each $n\ge 0$.
\end{lemma}
\begin{proof}
We prove this statement by induction on integer $n$. In the base case, due to Definition~\ref{perfect Gn}, it suffices to show that $\vdash G\rhd G$, which is an instance of the Reflexivity axiom.

For the induction step, note that $X\vdash G^s_n\rhd G^s_{n-1}$ by Lemma~\ref{Gn to Gn-1}. At the same time, $X\vdash G^s_{n-1}\rhd G$ by the induction hypothesis. Hence, $X\vdash G^s_n\rhd G$ by the Transitivity axiom.
\end{proof}
\begin{lemma}\label{infty = n}
There is $n\ge 0$, such that $G^s_\infty=G^s_n$.
\end{lemma}
\begin{proof}
Since $G^s_0\subseteq G^s_1\subseteq G^s_2\subseteq \dots\subseteq  V\cup\{\circlearrowleft\}$ and set $V$ is finite, there must exist an integer $n$ such that $G^s_n=\bigcup_{k\ge 0} G^s_k$. Therefore, $G^s_n=G^s_\infty$ by Definition~\ref{perfect Ginfty}.
\end{proof}
\begin{lemma}\label{Delta Ginfity}
 Set $(\Delta_{s\ldbrack h\rdbrack}(hd(h)))\setminus G^s_\infty$ is non-empty for each history $h$ such that $hd(h)\notin G^s_\infty$.
\end{lemma}
\begin{proof} 
Suppose $\Delta_{s\ldbrack h\rdbrack}(hd(h))\subseteq G^s_\infty$ for some history $h$. It suffices to show that $hd(h)\in G^s_\infty$. Indeed, by Lemma~\ref{infty = n} there is $n$ such that $G^s_\infty=G^s_n$. Thus, $\Delta_{s\ldbrack h\rdbrack}(hd(h))\subseteq G^s_n$. Hence, $hd(h)\in G^s_{n+1}$ by Definition~\ref{perfect Gn}. Then, $hd(h)\in G^s_\infty$ by Definition~\ref{perfect Ginfty}.
\end{proof}
\begin{lemma}\label{history extend}
For any positive integer $k$ and any history $w_0,i_1,w_1,\dots,w_{k-1}$, if $w_{k-1}\notin G^s_\infty$, then there is a state $w_{k}\notin G^s_\infty$ such that $w_0,i_1,w_1,\dots,w_{k-1},i_k,w_k$ is a history, where $i_k=s\ldbrack w_0,i_1,w_1,\dots,w_{k-1}\rdbrack$.
\end{lemma}
\begin{proof}
By Lemma~\ref{Delta Ginfity}, set $(\Delta_{i_k}(w_{k-1}))\setminus G^s_\infty$ is not empty. Let $w_k$ be any state such that
$w_k\in \Delta_{i_k}(w_{k-1})$ and $w_k\notin G^s_\infty$. Then, $w_0,i_1,w_1,\dots,w_{k-1},i_k,w_k$ is a history by Definition~\ref{history}.
\end{proof}
\begin{lemma}\label{path exists not in G infty}
For each state $w_0\notin G^s_\infty$ there exists a path $w_0,i_1,w_1,\dots$ under recall strategy $s$ such that $w_k\notin G^s_\infty$ for each $k\ge 0$.
\end{lemma}
\begin{proof}
Note that single-element sequence $w_0$ is a history by Definition~\ref{history}. Due to Lemma~\ref{history extend}, there is an infinite sequence $w_0,i_1,w_1,\dots$ such that for each integer $k\ge 1$,
\begin{enumerate}
    \item $w_0,i_1,w_1,\dots,w_{k-1}$ is a history,
    \item $i_k=s\ldbrack w_0,i_1,w_1,\dots,w_{k-1}\rdbrack$,
    \item $w_k\notin G^s_\infty$.
\end{enumerate}
By Definition~\ref{perfect path}, sequence $w_0,i_1,w_1,\dots$ is a path under recall strategy $s$.
\end{proof}
\begin{lemma}\label{perfect vdash to vDash}
If $X\vdash A\rhd B$, then $T(X)\vDash A\rhd B$.
\end{lemma}
\begin{proof}
Assumption $X\vdash A\rhd B$ implies $(A,B)\in I$ by Definition~\ref{perfect canonical I}. Consider recall strategy $s$ such that $s\ldbrack h\rdbrack=(A,B)$ for each class of histories $\ldbrack h \rdbrack$. Consider any path $w_0,i_1,w_1,\dots$ under recall strategy $s$ where $[w_0]\in A^*$. Then, $w_0\in A$ by Lemma~\ref{* iff}.

By Definition~\ref{perfect sat} and Definition~\ref{visit set} it suffices to show that $[w_1]\in B^*$. Indeed, $i_1=(A,B)$ by choice of recall strategy $s$. Thus, $\Delta_{i_1}(w_0)=\Delta_{(A,B)}(w_0)=B$ by Definition~\ref{perfect canonical delta} and due to the assumption $w_0\in A$. 

By Definition~\ref{perfect path}, sequence $w_0,i_1,w_1$ is a history. Hence, $w_1\in \Delta_{i_1}(w_0)$ by Definition~\ref{history}.
Thus, $w_1\in \Delta_{i_1}(w_0)=B$. Then, $[w_1]\in B^*$ by Lemma~\ref{* iff}.
\end{proof}
\begin{lemma}\label{perfect vDash to vdash}
If $T(X)\vDash E\rhd G$, then $X\vdash E\rhd G$.
\end{lemma}
\begin{proof}
By Definition~\ref{perfect sat}, assumption $T(X)\vDash E\rhd G$ implies  $Path_s(E)\subseteq Visit(G)$ for some recall strategy $s$.

\noindent{\em Case I:} $E\subseteq G^s_\infty$. By Lemma~\ref{infty = n} there exists an integer $n\ge 0$, such that $E\subseteq G^s_n$. Thus, $\vdash E\rhd G^s_n$ by the Reflexivity axiom. At the same time, $X\vdash G^s_n\rhd G$ by Lemma~\ref{Gn G}. Therefore, $X\vdash E\rhd G$ by the Transitivity axiom. 

\noindent{\em Case II:} $E\nsubseteq G^s_\infty$. Then, there is an element $w_0\in E$ such that $w_0\notin G^s_\infty$. Thus, by Lemma~\ref{path exists not in G infty} there is a path $\pi=w_0,i_1,w_1,\dots$ under recall strategy $s$ such that $w_k\notin G^s_\infty$ for all $k\ge 0$. Hence, $w_k\notin G$ for all $k\ge 0$ because $G=G^s_0\subseteq G^s_\infty$ by Definition~\ref{perfect Gn} and Definition~\ref{perfect Ginfty}.
Thus, $[w_0]\in E^*$ and $[w_k]\notin G^*$ for all $k\ge 0$ by Lemma~\ref{* iff}. Then,  states in path $\pi\in Path_s(E)$ by Definition~\ref{path set} and $\pi\notin Visit(G)$ by Definition~\ref{visit set}. Therefore, $Path_s(E)\nsubseteq Visit(G)$, which contradicts the choice of strategy $s$.
\end{proof}
\begin{lemma}\label{perfect induction}
$X\vdash\phi$ iff $T(X)\vDash\phi$ for each $\phi\in \Phi$.
\end{lemma}
\begin{proof}
We prove this lemma by induction on the structural complexity of $\phi$. The base case follows from Lemma~\ref{perfect vdash to vDash} and Lemma~\ref{perfect vDash to vdash}. The induction case follows from the maximality and the consistency of set $X$ in the standard way.
\end{proof}

We are now ready to state and prove the completeness theorem for the recall strategies.
\begin{theorem}\label{perfect completeness}
If $T\vDash\phi$ for every system $T$, then $\vdash\phi$.
\end{theorem}
\begin{proof}
Suppose $\nvdash\phi$. Let $X$ be a maximal consistent set containing formula $\neg\phi$. Thus, $T(X)\vDash\neg\phi$ by Lemma~\ref{perfect induction}. Therefore, $T(X)\nvDash\phi$.
\end{proof}

\section{Navigation with Memoryless Strategies}

In this section we give a sound and complete axiomatization of navigability under memoryless strategies. We start by modifying Definition~\ref{perfect sat} to refer to memoryless strategies instead of recall strategies:

\begin{definition}\label{imperfect sat}
$T\vDash A\rhd B$ if $Path_s(A)\subseteq Visit(B)$ for some memoryless strategy $s$ of a transition system $T$.
\end{definition}

\subsubsection{Axioms} 
The logical system for memoryless strategies is the same as for recall strategies with the exception that the Transitivity axiom is replaced by the following principle:

\begin{enumerate}
    \item[3.] Monotonicity: $A'\rhd B \to A\rhd B$, where $A\subseteq A'$.
\end{enumerate}
This principle can be derived from Armstrong's axioms.

\begin{theorem}\label{imperfect soundness}
If $\vdash\phi$, then $T\vDash\phi$ for every system $T$.
\end{theorem}
\begin{proof}
Soundness of the Reflexivity axiom and the Augmentation axiom is similar to the case of perfect recall, see Theorem~\ref{perfect soundness}. Soundness of the Monotonicity axiom follows from $Path_s(A)\subseteq Path_s(A')$, where $A\subseteq A'$.
\end{proof}

\subsubsection{Completeness}
In the rest of this section we prove completeness of our logical system with respect to the memoryless semantics. First, we define a canonical transition system $T(X)=(S,\sim, *, I,\{\Delta_i\}_{i\in I})$ for an arbitrary maximal consistent set of formulae $X\subseteq \Phi$. 

Like in the perfect recall case, the canonical system has one state for each view and an additional ``black hole" state $\circlearrowleft$. Unlike the previous construction, the new canonical transition system has more additional states besides state $\circlearrowleft$. Drawing on our original intuition of a transition system as a maze, we think about these new states as ``wormholes". For any sets of states $A$ and $B$ in the maze there is a wormhole state $w(A,B)$ that can be used to travel one-way from set $A$ to set $B$. Then, $S=V\cup\{\circlearrowleft\}\cup\{w(A,B)\;|\;A,B\subseteq V\}$.

The agent can distinguish any two different non-wormhole states, but she can not distinguish wormholes. In other words, each non-wormhole state $v\in V\cup \{\circlearrowleft\}$ forms its own indistinguishability class $[v]=\{v\}$, while all wormholes belong to the same single indistinguishability class of wormholes. Like in the perfect recall case, for each $v\in V$, we define $v^*$ to be class $[v]$.





\begin{lemma}\label{imperfect * iff}
$u\in A$ iff $[u]\in A^*$ for each view $u\in V$ and each set $A\subseteq V$. \qed
\end{lemma}
Like in the canonical model for the perfect recall case, for sets $A,B\subseteq V$ such that $X\vdash A\rhd B$, we introduce an instruction $(A,B)$ that can be used to navigate from set $A$ to set $B$. Unlike the perfect recall case, we introduce such an instruction only if sets $A$ and $B$ are disjoint. This is an insignificant technical restriction that we use to simplify the proof of Lemma~\ref{imperfect vDash to vdash}.

\begin{definition}\label{imperfect canonical I}
$$I=\{(A,B)\;|\;X\vdash A\rhd B \mbox{ and } A\cap B=\varnothing\}.$$
\end{definition}
Recall that assumption $X\vDash A\rhd B$ implies that sets $A$ and $B$ are nonempty due to Definition~\ref{Phi}.

In the perfect recall case, instruction $(A,B)$ can be used to transition the system {\em directly} from a state in set $A$ to a state in set $B$. In our case, this transition happens via the wormhole state $w(A,B)$. In other words, when instruction $(A,B)$ is invoked in a state from set $A$, the system transitions into state $w(A,B)$. When the same instruction is invoked in  $w(A,B)$, the system transitions into a state in $B$.
\begin{definition}\label{imperfect canonical Delta}
$$
\Delta_{(A,B)}(u)=
\begin{cases}
\{w(A,B)\}, & \mbox{if $u\in A$},\\
B, & \mbox{if $u=w(A,B)$},\\
\{\circlearrowleft\}, & \mbox{otherwise}.
\end{cases}
$$
\end{definition}

\begin{lemma}\label{imperfect vdash to vDash disjoint}
If $X\vdash A\rhd B$ and sets $A$ and $B$ are disjoint, then $T(X)\vDash A\rhd B$.
\end{lemma}
\begin{proof}
 Assumptions $X\vdash A\rhd B$ and $A\cap B=\varnothing$ imply that $(A,B)\in I$, by Definition~\ref{imperfect canonical I}. Consider a memoryless strategy $s$ such that $s(x)=(A,B)$ for each class $x$. By Definition~\ref{imperfect sat}, it suffices to show that $Path_s(A)\subseteq Visit(B)$.

Consider any $w_0,i_1,w_1,\dots\in Path_s(A)$. Then, by Definition~\ref{path set}, sequence $w_0,i_1,w_1,\dots$ is a path under strategy $s$ such that $[w_0]\in A^*$. Hence,  $w_0\in A$ by Lemma~\ref{imperfect * iff}. Thus, 
\begin{equation}\label{Delta w0}
    \Delta_{(A,B)}(w_0)=\{w(A,B)\}
\end{equation}
by Definition~\ref{imperfect canonical Delta}. 

At the same time, $w_1\in\Delta_{i_1}(w_0)$ by Definition~\ref{history}. Hence, $w_1\in\Delta_{s[w_0]}(w_0)$ by Definition~\ref{path}. Thus, $w_1\in\Delta_{(A,B)}(w_0)$ by the choice of strategy $s$. Then, $w_1=w(A,B)$ by equation~(\ref{Delta w0}). Hence,
\begin{equation}\label{Delta w1}
    \Delta_{(A,B)}(w_1)=B
\end{equation}
by Definition~\ref{imperfect canonical Delta}. 

Similarly, $w_2\in\Delta_{i_2}(w_1)$ by Definition~\ref{history}. Hence, $w_2\in\Delta_{s[w_1]}(w_1)$ by Definition~\ref{path}. Thus, $w_2\in\Delta_{(A,B)}(w_1)$ by the choice of strategy $s$. Then, $w_2\in B$ by equation~(\ref{Delta w1}). Hence, $[w_2]\in B^*$ by Lemma~\ref{imperfect * iff}. Therefore, $w_0,i_1,w_1,i_2,w_2,\dots\in Visit(B)$ by Definition~\ref{visit set}.
\end{proof}

\begin{lemma}\label{imperfect vdash to vDash}
If $X\vdash A\rhd B$, then $T(X)\vDash A\rhd B$.
\end{lemma}
\begin{proof}
Suppose that $X\vdash A\rhd B$. 

If $A\setminus B\neq \varnothing$. Thus, $X\vdash A\setminus B\rhd B$ by the Monotonicity Axiom. Hence, $T(X)\vDash A\setminus B\rhd B$ by Lemma~\ref{imperfect vdash to vDash disjoint}. Therefore, $T(X)\vDash A\rhd B$ due to the soundness of the Augmentation axiom, see Theorem~\ref{imperfect soundness}.

If $A\setminus B=\varnothing$, then $A\subseteq B$. Therefore, $T(X)\vDash A\rhd B$ due to the soundness of the Reflexivity axiom.
\end{proof}

Recall that all wormhole states belong to a single indistinguishability class of wormholes. For any memoryless strategy $s$, let $(A_s,B_s)$ be the instruction assigned by strategy $s$ to the class of wormholes. Once strategy $s$ is fixed, the states of the canonical transition system can be partitioned into five groups: set $A_s$, set $V\setminus A_s$, the single element set $\{\circlearrowleft\}$ containing the black hole state, the single element set $\{w(A_s,B_s)\}$ containing the wormhole state $w(A_s,B_s)$, and the set $\{w(C,D)\;|\; (C,D)\neq (A_s,B_s)\}$ of all other wormholes. Definition~\ref{imperfect canonical Delta} restricts transitions under strategy $s$ that are possible between these five groups of states. For example, from set $V\setminus A_s$ one can transition either into set $\{\circlearrowleft\}$ or into set $\{w(C,D)\;|\; (C,D)\neq (A_s,B_s)\}$. The arrows in Figure~\ref{trap figure} show all possible transitions between these five groups of states allowed under Definition~\ref{imperfect canonical Delta}. These five groups of states can be further classified into ``above the line" and ``below the line" states, as shown. Notice that once the system transitions into one of the ``below the line" states, it is trapped there and it will never be able to transition under the memoryless strategy $s$ into an ``above the line" state.

\begin{figure}[ht]
\begin{center}
\vspace{-0mm}
\scalebox{.7}{\includegraphics{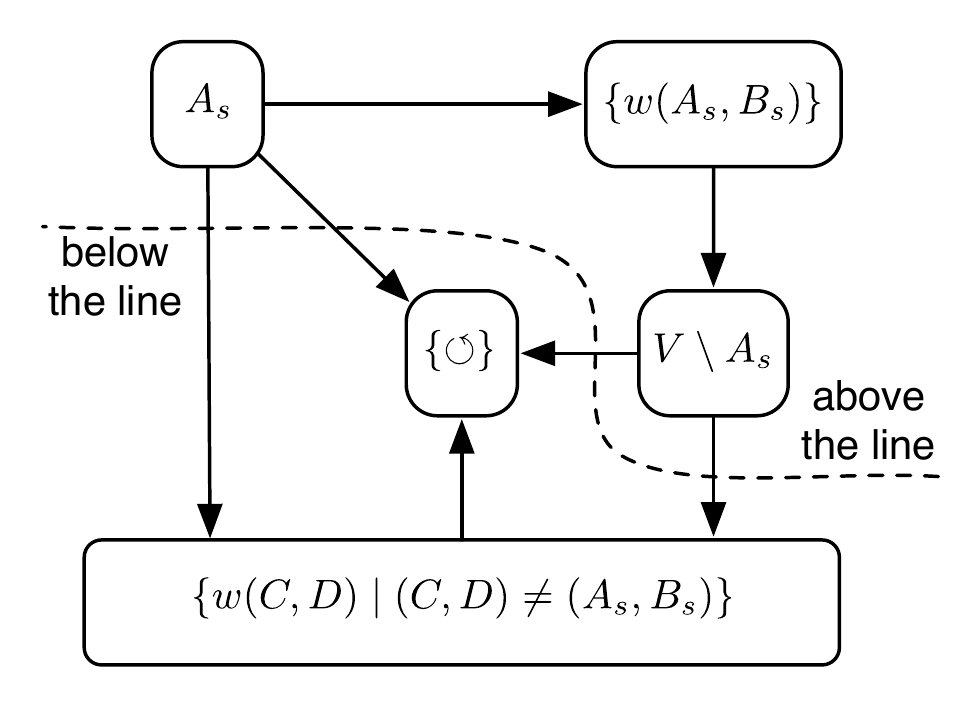}}
\caption{Transitions under memoryless strategy $s$.}\label{trap figure}
\end{center}
\vspace{0mm}
\end{figure}

\begin{lemma}\label{imperfect vDash to vdash}
If $T(X)\vDash E\rhd G$, then $X\vdash E\rhd G$.
\end{lemma}
\begin{proof}
If $E\subseteq G$, then $\vdash E\rhd G$ by the Reflexivity axiom. In the rest of the proof we suppose that there is $e_0\in E\setminus G$. By Definition~\ref{imperfect sat}, assumption $T(X)\vDash E\rhd G$ implies that there exists a memoryless strategy $s$ such that $Path_s(E)\subseteq Visit(G)$.

First, we show that $E\setminus G\subseteq A_s$. Suppose there is a view $e_1\in E$ such that $e_1\notin G$ and $e_1\notin A_s$. By Definition~\ref{history}, single-element sequence $e_1$ is a history. Thus, by Lemma~\ref{path exists}, there is a path $\pi\in Path_s(E)$ that starts with state $e_1$. Since path $\pi$ starts in a state from set $V\setminus A_s$, all non-initial states of this path are ``below the line", see Figure~\ref{trap figure}. Hence, neither of the states in path $\pi$ belong to set $G$, because $e_1\notin G$ and none of the states ``below the line" are in set $G$ either.     Thus, $\pi\in Path_s(E)$ and $\pi\notin Visit(G)$, which contradicts the choice of strategy $s$. Therefore, $E\setminus G\subseteq A_s$. In particular $e_0\in A_s$.

Second, we show that $s[e_0]=(A_s,B_s)$. Suppose that $s[e_0]=(C,D)$, where $(C,D)\neq (A_s,B_s)$. If $e_0\notin C$, then  $\pi=e_0,(C,D),\circlearrowleft,s[\circlearrowleft],\circlearrowleft,s[\circlearrowleft],\dots$ is a path under strategy $s$ by Definition~\ref{path} and Definition~\ref{imperfect canonical Delta}. Note that $\pi\notin Visit(G)$ because $e_0\notin G$ and $\pi\in Path_s(E)$. This contradicts $Path_s(E)\subseteq Visit(G)$. Similarly, if $e_0\in C$, then  sequence $\pi=e_0,(C,D),w(C,D),(A_s,B_s),\circlearrowleft,s[\circlearrowleft],\dots$ is a path such that $\pi\notin Visit(G)$ and $\pi\in Path_s(E)$, which again contradicts $Path_s(E)\subseteq Visit(G)$.

Third, we prove that $B_s\subseteq G$. Suppose that there is a state $b_0\in B_s\setminus G$. By Definition~\ref{history}, Definition~\ref{imperfect canonical Delta}, and the choice of instruction $(A_s,B_s)$, sequence $e_0,(A_s,B_s),w(A_s,B_s),(A_s,B_s),b_0$ is a history. Thus, by Lemma~\ref{path exists}, there is a path $\pi\in Path_s(E)$ that starts as $e_0,(A_s,B_s),w(A_s,B_s),(A_s,B_s),b_0$. Assumption $b_0\in B_s\setminus G$ implies that $b_0\in V\setminus A_s$ because sets $A_s$ and $B_s$ are disjoint by Definition~\ref{imperfect canonical I}. Thus, path $\pi$ after state $b_0$ contains only states ``below the line", see Figure~\ref{trap figure}, none of which are in set $G\subseteq V$. Recall also that $e_0,b_0\notin G$ by the choice of states $e_0$ and $b_0$. Thus, $\pi\in Path_s(E)$ and $\pi\notin Visit(G)$, which again contradicts $Path_s(E)\subseteq Visit(G)$.

Note that $X\vdash A_s\rhd B_s$ by Definition~\ref{imperfect canonical I}. Hence, $X\vdash (A_s\cup G)\rhd (B_s\cup G)$ by the Augmentation axiom. Thus, $X\vdash (A_s\cup G)\rhd G$ because $B_s\subseteq G$. At the same time, $E\setminus G\subseteq A_s$ implies that $E\subseteq A_s\cup G$. Therefore, $X\vdash E\rhd G$ by the Monotonicity axiom.

\end{proof}

\begin{lemma}\label{imperfect induction}
$X\vdash\phi$ iff $T(X)\vDash\phi$ for each $\phi\in \Phi$.
\end{lemma}
\begin{proof}
We prove this lemma by induction on the structural complexity of $\phi$. The base case follows from Lemma~\ref{imperfect vdash to vDash} and Lemma~\ref{imperfect vDash to vdash}. The induction case follows from the maximality and the consistency of set $X$ in the standard way.
\end{proof}

We are now ready to state and prove the completeness theorem for memoryless strategies.
\begin{theorem}\label{imperfect completeness}
If $T\vDash\phi$ for every system $T$, then $\vdash\phi$.
\end{theorem}
\begin{proof}
Suppose $\nvdash\phi$. Let $X$ be a maximal consistent set containing formula $\neg\phi$. Thus, $T(X)\vDash\neg\phi$ by Lemma~\ref{imperfect induction}. Therefore, $T(X)\nvDash\phi$.
\end{proof}

\section{Conclusion}
In this paper we have shown that the properties of navigability under perfect recall strategies are exactly those described by Armstrong's axioms for functional dependency in database theory. In the absence of perfect recall, the Transitivity axiom is no longer valid, but it could be replaced by the Monotonicity axiom.

\bibliographystyle{aaai}

\end{document}